\documentclass[11pt]{article}
\usepackage[preprint]{acl}
\usepackage{times}
\usepackage{latexsym}
\usepackage[T1]{fontenc}
\usepackage[utf8]{inputenc}
\usepackage{microtype}
\usepackage{inconsolata}
\usepackage{graphicx}
\usepackage{amsmath,amssymb,amsthm}
\usepackage{booktabs}
\usepackage{algorithm}
\usepackage{algpseudocode}
\newcommand{\gear}{GEAR}
\newcommand{\galr}{\gear{}}
\newcommand{\jcomm}{\textsc{J-Comm}}
\newcommand{\norm}[1]{\left\lVert #1\right\rVert_2}
\newcommand{\unit}[1]{\frac{#1}{\norm{#1}}}
\newcommand{\R}{\mathbb{R}}

\newtheorem{theorem}{Theorem}

\title{J-Comm: Jacobian-Guided Latent Communication between LLM Agents}
\author{
  \textbf{Xiangwei Wang$^{1}$ \quad Wei Wang$^{1}$ \quad Sachith Seneviratne$^{1}$ \quad Ken Chen$^{1}$} \\
  \textbf{Yu Xia$^{1}$ \quad Peng Wang$^{2}$ \quad Saman Halgamuge$^{1}$} \\[0.5em]
  $^{1}$The University of Melbourne \quad $^{2}$Shanghai Jiao Tong University
}

\begin{document}
\maketitle

\begin{abstract}
Latent communication requires both an effective continuous reasoning process and an interface through which an external model can leverage its intermediate computation. We present \jcomm{}, a framework comprising two core components. First, A training-free \textbf{GE}ometric \textbf{A}greement for Latent \textbf{R}easoning (\gear{}) reweights token candidates over existing model embeddings, which monotonically optimizes geometric consensus among candidate tokens in embedding space at each step. Second, we introduce an H+J (Hidden-and-Jacobian) communication interface that gives the Receiver two complementary views of the Sender’s reasoning. The H channel supplies continuous inputs derived from feedback distributions, while the J channel conveys intermediate-state cues in vocabulary-indexed Jacobian-lens coordinates. Intuitively, H provides inputs for the Receiver to reason from, while J offers additional guidance within its computation. The Receiver reconstructs these cues using its own dictionary and injects them into intermediate layers during prefill, enabling communication without sharing key-value caches or requiring identical model widths. In Qwen3-8B $\rightarrow$ Qwen3-4B communication, \jcomm{} achieves the highest mean score across seven benchmarks, including 94.80\% on MATH500, 54.55\% on GPQA-Diamond, 73.33\% on AIME2024+2025, and 62.93\% success on ALFWorld. Compared with TextMAS, it improves performance by 1.29--8.59 percentage points and reduces generated text tokens by 19.2--74.8\%.

\end{abstract}

\section{Introduction}
\label{sec:introduction}

Chain-of-thought prompting elicits intermediate reasoning steps in text \citep{wei2022chain}.
Multi-agent systems use text to exchange and revise responses through debate or conversation \citep{du2024debate,wu2024autogen}.
Latent communication lets agents pass probability-weighted embeddings \citep{pham2024cipher}, learned representations \citep{du2025interlat,zheng2025thoughtcomm}, or shared working memory \citep{zou2025latentmas}.
StateBridge aligns message states with input embeddings through a closed-form map \citep{peng2026statebridge}.
Here we consider both how the Sender constructs its continuous reasoning steps and which parts of that computation the Receiver receives.
We study these choices together: the inputs used to advance a reasoning trajectory also determine the messages passed between models.

\begin{figure*}[t]
  \centering
  \includegraphics[width=\textwidth]{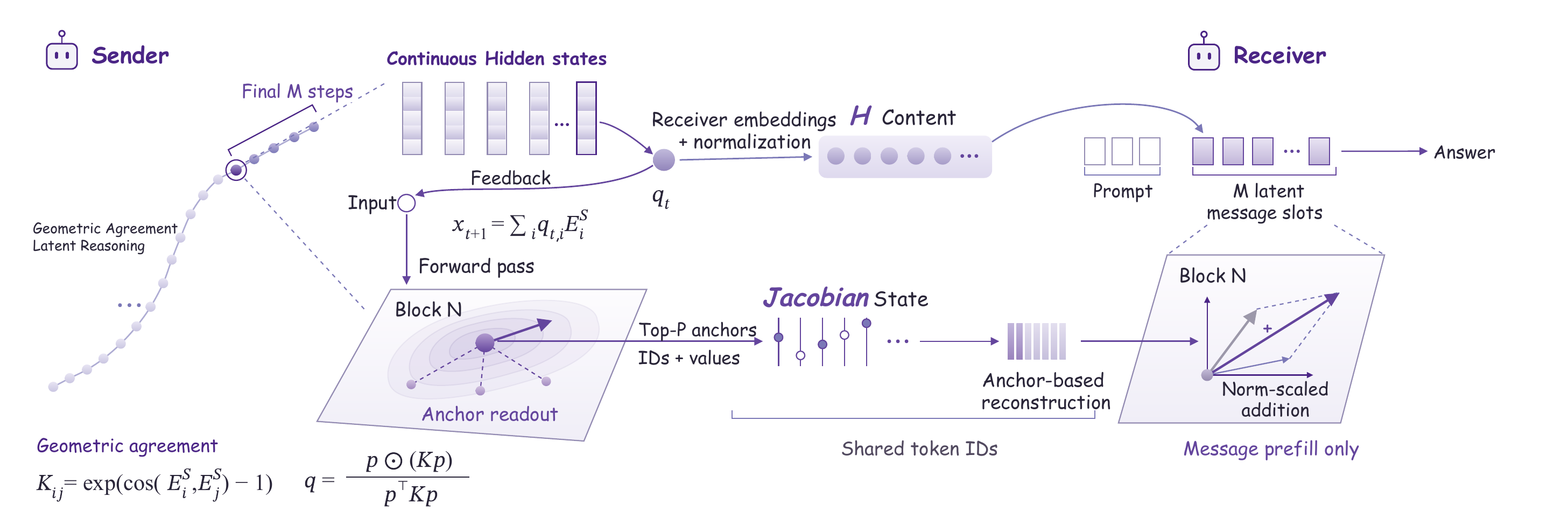}
  \caption{Overview of \jcomm{}.
  \galr{} reweights candidate probabilities by geometric agreement and feeds the resulting embedding mixture into the Sender's next position.
  The final $M$ consumed steps supply two messages: $H$ maps their feedback distributions into Receiver input embeddings, while $J$ carries token-anchor IDs and signed intermediate-state readouts for reconstruction in the Receiver's own dictionary.
  The $J$ update is added at the $M$ message positions during prefill.
  \emph{Top-P} denotes a fixed number of anchors ($P=32$), distinct from nucleus sampling.
  The two displayed blocks use index 26 in our main setting.}
  \label{fig:framework}
\end{figure*}

Coconut feeds the last hidden state back as the next input \citep{hao2025coconut}.
Soft Thinking instead constructs a continuous input by averaging candidate token embeddings under the model's predictive distribution \citep{zhang2025softthinking}.
The weights do not explicitly reflect whether the candidate embeddings point in similar directions.
The same probability profile can place mass on candidate sets with different geometries.
\textbf{GE}ometric \textbf{A}greement for Latent \textbf{R}easoning (\gear{}) uses this geometry when constructing the next input.
Each candidate receives support from the candidate set, weighted by the candidates' probabilities and embedding similarities.
We multiply its probability by this support, normalize the weights, and feed back a mixture of the original token embeddings.
The rule uses the frozen model's existing embedding table and requires no trained feedback module.

\jcomm{} constructs its messages from the steps that the Sender actually consumes.
For each retained step, we pair the feedback distribution with the intermediate state produced after its mixture enters the model.
The $H$ channel applies the same weights to the Receiver's token embeddings, following the receiver-side embedding construction used in CIPHER \citep{pham2024cipher}.
This gives the Receiver an input representation of the consumed distribution.
The $J$ channel exposes selected readouts of the paired intermediate state at a later layer.
Both channels thus refer to the same Sender step and act at different depths in the Receiver.

Prior work also communicates by combining or replacing intermediate activations \citep{ramesh2025activations}.
Our $J$ readouts use the Jacobian lens \citep{gurnee2026verbalizable}.
Shared token IDs identify anchors in the Sender's and Receiver's own dictionaries.
The Sender transmits the selected IDs and signed readouts; the Receiver reconstructs a vector with its own dictionary and adds a norm-scaled residual update during prefill.
The protocol allows different hidden widths, while keeping each model's KV cache local.
Its correspondence is defined by shared token identities and does not require matching activation coordinates.

Our contributions are:
\begin{enumerate}
  \item \textbf{\gear{}}, a continuous reasoning recurrence that uses probability-weighted geometric support to construct the next input from candidate token embeddings.
  \item \textbf{An $H+J$ communication protocol} that pairs consumed feedback distributions with intermediate token-coordinate readouts and forms $H$ inputs and $J$ residual writes using the Receiver's own embeddings and Jacobian-lens dictionary.
\end{enumerate}

\section{Background and Problem Setting}
\label{sec:background}

\subsection{Latent Reasoning and Communication}

Let $S$ and $R$ denote a Sender and a Receiver, with hidden widths $d_S$ and $d_R$.
Both agents receive the task input with their own role prompts.
The Sender performs continuous reasoning; the Receiver receives a sequence of latent messages and generates the answer.
We first define the interface for a shared token vocabulary $\mathcal V$ with corresponding token IDs; Appendix~\ref{app:cross-family} extends the correspondence across vocabularies.
The models' embedding vectors and hidden widths may differ.

Write $E^m_v\in\R^{d_m}$ for token $v$'s input embedding in model $m\in\{S,R\}$, and $u^m_v$ for its unembedding vector.
At iteration $t$, the Sender's final normalized state $h^S_t$ produces logits $z_t=U^S h^S_t$.
A continuous reasoning rule converts these logits or states into an embedding $x^S_{t+1}$, which is consumed at the next position while updating the Sender's local key--value (KV) cache.
This increases inference-time computation without generating a discrete reasoning token at each step.

The communication interface is distinct from this recurrence.
We retain $M$ steps from a trajectory of length $T$ and construct $M$ Receiver input vectors, together with $M$ intermediate-layer messages.
Thus, the number of communicated positions can remain fixed while Sender computation varies.

\subsection{Jacobian-Lens Coordinates}
\label{sec:jspace}

The Jacobian lens relates an intermediate residual stream to the model's later computation \citep{gurnee2026verbalizable}.
For model $m$, let $r^m_{\ell,a}$ be the residual at layer $\ell$ and position $a$, and $r^m_{L,b}$ the pre-final-normalization residual at a later position $b$.
An averaged Jacobian has the form
\begin{equation}
 \mathcal J^m_\ell
 =\mathbb E_{c,a,b\geq a}\!
 \left[\frac{\partial r^m_{L,b}}{\partial r^m_{\ell,a}}\right],
 \label{eq:jacobian}
\end{equation}
where $c$ ranges over calibration contexts.
Combining this map with token unembedding vectors yields directions associated with the model's disposition to verbalize particular tokens.
The original work studies sparsely active directions in this dictionary as J-space and reports evidence of workspace-like functions.
Here we use the dictionary to define token-indexed measurements of an intermediate state.

\paragraph{Dictionary used in our implementation.}
Our cached lenses target the pre-final-RMSNorm residual.
We fold the final RMSNorm gain $\gamma_m$ into each unembedding vector and define a unit anchor
\begin{equation}
 b^m_{\ell,v}
 =\unit{(\mathcal J^m_\ell)^\top
       \operatorname{diag}(\gamma_m)u^m_v}.
 \label{eq:anchor}
\end{equation}
The token ID $v$ names the coordinate; the vector $b^m_{\ell,v}$ is specific to the model and layer.
These anchors are generally nonorthogonal.
For a selected support $I$, the local communication space is
\begin{equation}
 \mathcal S^m_{\ell,I}
 =\left\{\sum_{v\in I}a_v b^m_{\ell,v}:a_v\in\R\right\}.
\end{equation}
Each message selects a small set of anchors from this dictionary.
The transmitted values are inner products with these anchors, rather than coefficients in an anchor expansion.
Their nonorthogonality matters when reconstructing a vector from the readouts.

\section{Geometric Agreement for Latent Reasoning}
\label{sec:galr}

\subsection{Candidates and Geometric Support}

We first apply temperature scaling and candidate filtering to $z_t$, obtaining a candidate set $\mathcal C_t$ and normalized probabilities $p_{t,i}$.
Soft Thinking would feed back $\sum_i p_{t,i}E^S_i$.
\galr{} additionally evaluates agreement within this candidate set.
Let $\widehat E^S_i=E^S_i/\norm{E^S_i}$.
We define
\begin{align}
 K_{t,ij}
 &=\exp\!\left((\widehat E^S_i)^\top
                 \widehat E^S_j-1\right),
 \label{eq:kernel}\\
 s_{t,i}&=\sum_{j\in\mathcal C_t}K_{t,ij}p_{t,j}.
 \label{eq:support}
\end{align}
An embedding receives more support when it is close to candidates that carry substantial probability mass.
The diagonal is included: $K_{t,ii}=1$.
Similarity is computed using normalized embeddings, while the feedback vector below uses the original embeddings.

\subsection{Agreement-Weighted Feedback}

We combine the predictive probability and geometric support multiplicatively:
\begin{equation}
 \boxed{q_t=
 \frac{p_t\odot(K_t p_t)}{p_t^\top K_t p_t}.}
 \label{eq:q}
\end{equation}
All kernel entries are positive, so $q_t$ is a normalized distribution with the same nonzero support as $p_t$.
The next continuous input is
\begin{equation}
 \boxed{x^S_{t+1}=
 \sum_{i\in\mathcal C_t}q_{t,i}E^S_i.}
 \label{eq:feedback}
\end{equation}
The Sender consumes this vector through its usual transformer forward pass and updates its own KV cache.
No sampled token or learned feedback module is required.
We call Equation~\ref{eq:q} the \emph{geometric agreement feedback} (GAF) rule, and the resulting recurrent reasoning procedure \galr{}.

\paragraph{Interpretation.}
For $p_{t,i}>0$, the ratio $q_{t,i}/p_{t,i}$ equals the candidate's support divided by the probability-weighted mean support.
Candidates with above-average support therefore gain relative weight.
Because self-support contributes $p_{t,i}^2$ to the unnormalized weight, the rule combines probability sharpening with interactions among distinct candidates.
It uses a fixed cosine kernel rather than learned query, key, or value projections.
For $k=|\mathcal C_t|$ candidates, constructing the kernel costs $O(k^2d_S)$; our configuration uses at most 10 candidates.

\subsection{One-Step Geometric Selectivity}
\label{sec:monotone}

Fix the candidate set and suppress the step index $t$. Its geometric agreement is
\begin{equation}
 C(p)=p^\top Kp=\mathbb E_{i,j\sim p}[K_{ij}].
 \label{eq:agreement}
\end{equation}
GAF is a Baum--Eagon growth transformation: $C(q)\ge C(p)$, with equality exactly when $q=p$ \citep{baum1967inequality}.
It is also a symmetric discrete-time replicator update \citep{hofbauer1998evolutionary,losert1983dynamics}.
Appendix~\ref{app:gaf} gives the supporting derivation and an agreement-based bound on the change in the consumed embedding mixture.
Theorem~\ref{thm:gaf} characterizes how target affinity and agreement among the other candidates determine the target's probability gain in one GAF step; the full statement and proof appear in Appendix~\ref{app:gaf}.

\subsection{Adaptive Computation and Paired States}

The Sender runs for at least 40 consumed latent steps, independently of the communicated suffix length $M$.
After this minimum, it stops before consuming a step whose highest-weight $q$ candidate is the end-of-thinking token, or after sustained low predictive entropy; a hard cap bounds computation.
The exact policy appears in Appendix~\ref{app:settings}.

We retain the final $M$ \emph{consumed} feedback distributions.
For each $q_t$, we also retain $r^S_{\ell_S,t+1}$, the intermediate residual produced when $x^S_{t+1}$ is consumed.
Thus $q_t$ records the weights used to construct $x^S_{t+1}$, and $r^S_{\ell_S,t+1}$ records the resulting intermediate computation.
The retained pairs are
\begin{equation}
 \mathcal T=\{(q_t,r^S_{\ell_S,t+1})\}_{t=T-M}^{T-1}.
 \label{eq:trajectory}
\end{equation}
In the communication setting, the Sender hands off this trajectory without generating a text answer.

\section{Dual-Channel Latent Communication}
\label{sec:communication}

The retained pairs in Equation~\ref{eq:trajectory} provide the messages for both channels (Figure~\ref{fig:framework}).
$H$ maps the consumed feedback distribution to the Receiver's input embeddings, where it is processed from the first layer.
$J$ carries token-indexed measurements of the intermediate state produced by that input and writes a reconstructed vector at a selected later layer.
The Receiver therefore receives an input constructed from the feedback weights and readouts of the computation they induced.

\subsection{$H$: Transporting the Consumed Distribution}
\label{sec:h}

For each retained $q_t$, we use the corresponding Receiver embeddings:
\begin{align}
 \bar e^R_t&=\sum_{i\in\mathcal C_t}q_{t,i}E^R_i,\\
 H_t&=r_E^R\frac{\bar e^R_t}{\norm{\bar e^R_t}},
 \label{eq:h}
\end{align}
where $r_E^R$ is the fixed mean norm of Receiver input embeddings.
The $H$ channel is $[H_{T-M},\ldots,H_{T-1}]$.
The same probability weights thus construct inputs in the two models' respective embedding spaces, without requiring $d_S=d_R$.
Unlike the Sender feedback in Equation~\ref{eq:feedback}, this cross-model input is norm calibrated.

The $H$ payload is an expected-embedding representation of $q_t$.
This receiver-side construction follows CIPHER \citep{pham2024cipher}, with the mixture weights supplied by the consumed \galr{} distribution and the norm calibrated to the Receiver embedding table.

\subsection{$J$: Reading Intermediate Token Coordinates}
\label{sec:jread}

The $J$ channel uses the fixed Sender dictionary from Equation~\ref{eq:anchor}.
For each retained intermediate state, we compute the signed normalized readout
\begin{equation}
 c_t(v)=
 (b^S_{\ell_S,v})^\top
 \frac{r^S_{\ell_S,t+1}}{\norm{r^S_{\ell_S,t+1}}}.
 \label{eq:readout}
\end{equation}
We select $I_t$, the $P=32$ shared nonspecial token IDs with the largest readout values.
We rank the signed readouts by their values and retain these values directly.
The message is the list
\begin{equation}
 J_t=\{(v,c_t(v)):v\in I_t\}.
 \label{eq:jpayload}
\end{equation}
The anchor dictionary is fixed after calibration, while the selected IDs and their values vary across steps and questions.
These anchors are chosen from the eligible vocabulary, independently of the small candidate set used by \galr{}.

\subsection{Reconstruction in the Receiver's Dictionary}
\label{sec:jwrite}

Using the same IDs $I_t$, the Receiver retrieves its own anchor rows to form $B^R_{\ell_R,I_t}\in\R^{P\times d_R}$.
Let $c_t\in\R^P$ stack the transmitted readouts in the same order.
We reconstruct a vector using a regularized dual dictionary:
\begin{align}
 D^R_{\ell_R,I_t}
 &=\left(B^R_{\ell_R,I_t}(B^R_{\ell_R,I_t})^\top
          +\lambda I_P\right)^{-1}B^R_{\ell_R,I_t},
 \label{eq:dual}\\
 d^R_{\ell_R,t}&=(D^R_{\ell_R,I_t})^\top c_t.
 \label{eq:reconstruct}
\end{align}
Here $\lambda=10^{-3}$.
Equivalently, $d^R_{\ell_R,t}$ minimizes
$\norm{B^R_{\ell_R,I_t}d-c_t}^2+\lambda\norm{d}^2$.
This accounts for overlap among anchor directions; directly summing the anchors would generally mix their readouts.
The token IDs select corresponding rows in the two dictionaries; the Receiver's anchor geometry determines the reconstructed vector.

\paragraph{Magnitude calibration.}
To calibrate injection strength across layer-specific dictionaries, we compute
\begin{equation}
 \rho_t=\norm{d^R_{\ell_{\rm ref},t}},\qquad
 \Delta^R_t=\rho_t\frac{d^R_{\ell_R,t}}{\norm{d^R_{\ell_R,t}}},
 \label{eq:radius}
\end{equation}
using a fixed Receiver reference lens at $\ell_{\rm ref}=29$.
The reference and target reconstructions use the same IDs and readouts.
The reference layer calibrates magnitude; the actual write is at $\ell_R$.

\paragraph{Prefill-only write.}
The Receiver input consists of its question and role prompt, the $M$ vectors $H_t$, and an answer-generation cue.
At the output of block $\ell_R$, each corresponding message-position residual $r^R_{\ell_R,t}$ is updated as
\begin{equation}
 \boxed{r^R_{\ell_R,t}\leftarrow
 r^R_{\ell_R,t}+\alpha\norm{r^R_{\ell_R,t}}\Delta^R_t,}
 \label{eq:injection}
\end{equation}
with $\alpha=1$.
Here $t$ labels the Receiver message slot paired with Sender step $t$, rather than an absolute sequence position.
The update is added to the residual already produced by the Receiver.
Only the message positions are edited during prefill.
The Receiver then completes prefill and decodes normally from its own KV cache, with no further $J$ writes during answer generation.

\paragraph{Finite-write analysis.}
Theorem~\ref{thm:j-fidelity} gives the selected-anchor readback and finite logit response, and characterizes output-order preservation under an endpoint-error bound. Appendix~\ref{app:j-fidelity} provides the full statement and proof.

\begin{algorithm}[t]
\caption{\jcomm{} with \galr{} and $H+J$}
\label{alg:jcomm}
\begin{algorithmic}[1]
\Require Frozen $S,R$; calibrated dictionaries; task input
\State Prefill $S$ with its role prompt and the input
\While{the stopping condition has not been met}
  \State Obtain candidate distribution $p_t$ from $S$
  \State Compute $K_t$ and $q_t$ (Eqs.~\ref{eq:kernel}, \ref{eq:q})
  \State Consume $\sum_i q_{t,i}E^S_i$; update local cache
  \State Retain $(q_t,r^S_{\ell_S,t+1})$ in the tail window
\EndWhile
\For{each of the final $M$ retained pairs}
  \State Construct $H_t$ in Receiver embeddings
  \State Read the top $P$ token anchors to obtain $J_t$
  \State Reconstruct and calibrate $\Delta^R_t$ in $R$
\EndFor
\State Prefill $R$ with its prompt, $H$, and answer cue
\Statex \hspace{\algorithmicindent}while applying Eq.~\ref{eq:injection} at message positions
\State Decode the answer using $R$'s local KV cache
\end{algorithmic}
\end{algorithm}

\section{Experiments}
\label{sec:experiments}

We evaluate geometric feedback, dual-channel communication, and transfer across model pairs.
All experimental results are averaged over three runs.
\emph{GAF} in the tables denotes the geometric feedback rule used by \gear{}.

\subsection{Experimental Setup}

\paragraph{Tasks and budgets.}
The primary pair is Qwen3-8B $\to$ Qwen3-4B.
We evaluate GSM8K~\citep{cobbe2021gsm8k} (1,319 questions), MATH500~\citep{lightman2024verify} (500), GPQA-Diamond~\citep{rein2024gpqa} (198), MedQA-USMLE four-option test~\citep{jin2021medqa} (1,273), EvalPlus MBPP+ v0.2.0~\citep{liu2023evalplus} (378), AIME2024+2025~\citep{h4_2025aime2024,lin2025aime2025} (60), and ALFWorld~\citep{shridhar2021alfworld} (134 unseen tasks).
Sender and Receiver text output caps are specified separately for each dataset in Appendix~\ref{app:task-budgets} (Table~\ref{tab:text-budgets}).
ALFWorld uses a 2,048-token cap per agent per action turn, with at most 50 turns per episode.
The adaptive latent Sender has a 32,768-step cap and communicates a suffix of at most $M$ consumed steps, with $M=400$ in the primary configuration.

\paragraph{Inference and scoring.}
Continuous Sender computation uses Hugging Face, and the primary Receiver uses vLLM 0.10.2 V0.
The primary Qwen configuration reads and writes block 26, with frozen models and fixed Jacobian dictionaries; calibration and execution details appear in Appendix~\ref{sec:implementation}.
Model execution uses bfloat16.
Text decoding uses temperature 0.6, top-$k=20$, and top-$p=0.95$; latent candidate filtering follows Appendix~\ref{app:settings}.
Receiver backends and output caps are matched within each comparison.
TextMAS uses cached HF Sender plans or vLLM-generated plans according to the configuration.

Mathematical answers use answer-equivalence scoring; multiple-choice answers use the final option letter; MBPP+ reports pass@1 on both base and additional EvalPlus tests.
Length-truncated answers are scored as incorrect.
The multiple-choice parser accepts a boxed option letter, optionally followed by option text, with the same scoring rule for every arm.
The feedback figure reports question-paired 95\% bootstrap intervals; the layer and communication-length figures report standard deviations across runs.
Sender latent steps and generated text tokens are accounted for separately.

\paragraph{Controls.}
Receiver-only answers without a message, while TextMAS supplies the Sender's generated plan.
Identity40 uses 40 latent recurrence steps as proposed by LatentMAS~\cite{zou2025latentmas},  whereas Identity($T_{\rm GAF}$) matches each question's GAF step count.
Both use Receiver block 26 in the primary $H+J$ comparisons.
The paired $H$ and $H+J$ ablations reuse the same Sender trajectory.
Null-slot controls use vocabulary-mean Receiver embeddings at the ordinary message radius; $J$-only adds the matched $J$ update to these slots.
Shuffled-$J$ controls retain the recipient's $H$ and intervention radius while taking $J$ from a different question.

\subsection{Primary Cross-Scale Results}

\begin{table*}[t]
\centering\small\setlength{\tabcolsep}{3.5pt}
\resizebox{\textwidth}{!}{\begin{tabular}{@{}lrrrrrrr@{}}
\toprule
Method & GSM8K & MATH500 & GPQA-Diamond & MedQA & MBPP+ & AIME24+25 & ALFWorld \\
\midrule
Text MAS & 93.25 {\scriptsize (ref.)} & 93.20 {\scriptsize (ref.)} & 45.96 {\scriptsize (ref.)} & 78.55 {\scriptsize (ref.)} & 71.16 {\scriptsize (ref.)} & 68.33 {\scriptsize (ref.)} & 61.39 {\scriptsize (ref.)} \\
Receiver only & 86.58 {\scriptsize ($\downarrow$36.5\%)} & 88.20 {\scriptsize ($\downarrow$32.6\%)} & 39.90 {\scriptsize ($\downarrow$22.5\%)} & 69.76 {\scriptsize ($\downarrow$20.8\%)} & 72.22 {\scriptsize ($\downarrow$27.6\%)} & 67.78 {\scriptsize ($\uparrow$0.6\%)} & 49.25 {\scriptsize ($\downarrow$55.8\%)} \\
Identity40+$H+J26$ & 86.73 {\scriptsize ($\downarrow$36.6\%)} & 89.40 {\scriptsize ($\downarrow$32.0\%)} & 39.39 {\scriptsize ($\downarrow$22.5\%)} & 70.23 {\scriptsize ($\downarrow$25.4\%)} & 71.16 {\scriptsize ($\downarrow$26.7\%)} & 69.44 {\scriptsize ($\uparrow$9.1\%)} & 53.73 {\scriptsize ($\downarrow$56.4\%)} \\
Identity($T_{\rm GAF}$)+$H$ & 87.19 {\scriptsize ($\downarrow$38.7\%)} & 89.00 {\scriptsize ($\downarrow$34.4\%)} & 39.39 {\scriptsize ($\downarrow$27.0\%)} & 71.25 {\scriptsize ($\downarrow$26.9\%)} & 71.69 {\scriptsize ($\downarrow$29.5\%)} & 65.00 {\scriptsize ($\uparrow$1.4\%)} & 54.25 {\scriptsize ($\downarrow$58.1\%)} \\
Identity($T_{\rm GAF}$)+$H+J26$ & 87.41 {\scriptsize ($\downarrow$38.8\%)} & 89.00 {\scriptsize ($\downarrow$35.3\%)} & 43.43 {\scriptsize ($\downarrow$26.0\%)} & 70.15 {\scriptsize ($\downarrow$27.5\%)} & 73.02 {\scriptsize ($\downarrow$31.0\%)} & 71.11 {\scriptsize ($\downarrow$3.6\%)} & 57.32 {\scriptsize ($\downarrow$60.3\%)} \\
GAF+$H$ & 92.12 {\scriptsize ($\downarrow$47.5\%)} & 93.60 {\scriptsize ($\downarrow$54.5\%)} & 49.49 {\scriptsize ($\downarrow$48.4\%)} & 78.40 {\scriptsize ($\downarrow$54.1\%)} & 73.54 {\scriptsize ($\downarrow$35.9\%)} & 72.78 {\scriptsize ($\downarrow$8.0\%)} & 59.70 {\scriptsize ($\downarrow$66.2\%)} \\
GAF+$H+J26$ & \textbf{94.54} {\scriptsize ($\downarrow$55.0\%)} & \textbf{94.80} {\scriptsize ($\downarrow$54.4\%)} & \textbf{54.55} {\scriptsize ($\downarrow$44.2\%)} & \textbf{80.52} {\scriptsize ($\downarrow$58.3\%)} & \textbf{76.72} {\scriptsize ($\downarrow$58.3\%)} & \textbf{73.33} {\scriptsize ($\downarrow$19.2\%)} & \textbf{62.93} {\scriptsize ($\downarrow$74.8\%)} \\
\bottomrule
\end{tabular}
}
\caption{Qwen3-8B $\to$ Qwen3-4B results averaged over three runs. Cells report accuracy, pass@1, or ALFWorld success rate (\%), followed by the percentage change in generated text tokens relative to TextMAS ($\downarrow$: fewer tokens; $\uparrow$: more tokens). TextMAS includes Sender thinking and plan generation plus Receiver output; latent methods generate text at the Receiver. ALFWorld token totals include all decisions in each task. Bold marks the highest mean score in each column.}
\label{tab:main-results}
\end{table*}

GAF+$H+J26$ achieves the highest mean score among the compared methods on all seven benchmarks (Table~\ref{tab:main-results}).
It reaches 94.54\% on GSM8K, 94.80\% on MATH500, 54.55\% on GPQA-Diamond, 80.52\% on MedQA, and 76.72\% on MBPP+.
The corresponding TextMAS scores are 93.25\%, 93.20\%, 45.96\%, 78.55\%, and 71.16\%, giving gains of 1.29, 1.60, 8.59, 1.97, and 5.56 percentage points.
On AIME2024+2025, GAF+$H+J26$ obtains 73.33\%, compared with 72.78\% for GAF+$H$ and 68.33\% for TextMAS, a gain of 5.00 percentage points over text communication.

The complete method also exceeds both Identity40+$H+J26$ and the step-matched Identity($T_{\rm GAF}$)+$H+J26$ across all seven tasks.
On MATH500, for example, these two controls achieve 89.40\% and 89.00\%, compared with 94.80\% for GAF+$H+J26$.
The reported text-token reductions for GAF+$H+J26$ range from 19.2\% on AIME to 74.8\% on ALFWorld; on the five non-AIME static tasks, they range from 44.2\% to 58.3\%.

On ALFWorld, GAF+$H+J26$ achieves 62.93\% success, followed by TextMAS at 61.39\% and GAF+$H$ at 59.70\%.
Receiver-only and Identity40+$H+J26$ obtain 49.25\% and 53.73\%; the step-matched Identity+$H$ and Identity+$H+J26$ controls reach 54.25\% and 57.32\%.
The complete method improves over TextMAS by 1.54 percentage points and over GAF+$H$ by 3.23 points.

\subsection{Effect of Geometric Feedback}
\label{sec:feedback-controls}

We compare GAF with Soft Thinking ($q=p$) and self-only sharpening ($q_i\propto p_i^2$), matching each control to the corresponding GAF trajectory length and candidate filter.
All three methods use the final 40 positions, $H+J26$, and an 8,192-token Receiver cap.
Figure~\ref{fig:feedback-gains} compares the Sender feedback rules under this shared communication configuration.

\begin{figure}[t]
\centering
\includegraphics[width=\columnwidth]{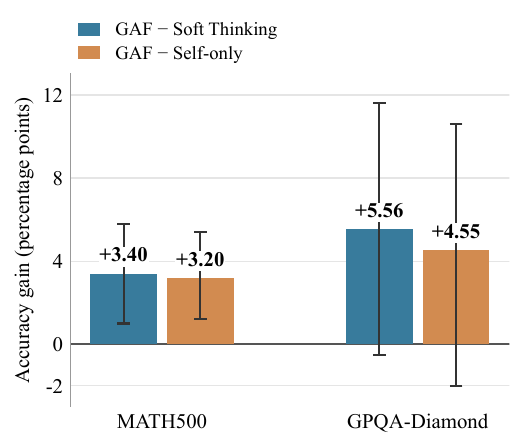}
\caption{Mean GAF accuracy gains over matched-step Soft Thinking and Self-only across three runs on MATH500 and GPQA-Diamond, with a fixed $H+J26$ interface. Error bars show question-paired 95\% bootstrap intervals.}
\label{fig:feedback-gains}
\end{figure}

On MATH500, GAF improves accuracy by 3.40 percentage points over Soft Thinking and 3.20 over Self-only.
The corresponding gains on GPQA-Diamond are 5.56 and 4.55 points.
GAF outperforms both the unmodified probability mixture and self-only sharpening, supporting the use of interactions among candidate embeddings in the feedback rule.

\subsection{Receiver Injection Layer}
\label{sec:j-layer-ablation}

\begin{figure}[t]
\centering
\includegraphics[width=\columnwidth]{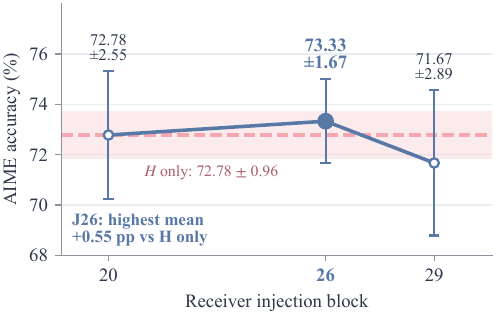}
\caption{Receiver-layer ablation on AIME2024+2025 with the same 400-position $H$ message. Points and error bars show mean $\pm$ SD across three runs; the dashed line and band show the $H$-only mean $\pm$ SD. The $H$-only and J26 results also appear in Table~\ref{tab:main-results}.}
\label{fig:j-layer-ablation}
\end{figure}

We hold the Sender readout block at 26 and compare Receiver write blocks 20, 26, and 29 (zero-based indices), together with an $H$-only control.
The three $J$ arms share the Sender trajectory, top-32 anchor IDs and readouts, and per-position relative intervention radius; each reconstructs the direction with its target-layer dictionary.
The write is applied to the 400 message positions during prefill.

Block 26 gives the highest mean accuracy, 73.33\%, compared with 72.78\% at block 20 and without $J$, and 71.67\% at block 29 (Figure~\ref{fig:j-layer-ablation}).

\subsection{Communication Length}
\label{sec:length-selection}

\begin{figure}[t]
\centering
\includegraphics[width=\columnwidth]{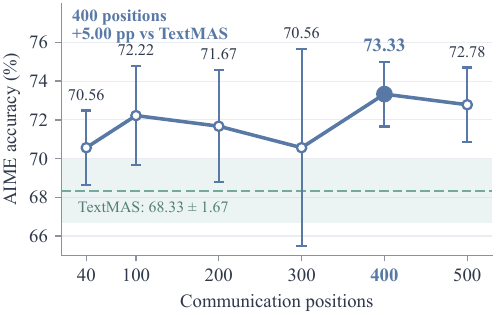}
\caption{Communication-length ablation on AIME2024+2025 using GAF+$H+J26$ and a shared Sender cache. Points and error bars show mean $\pm$ SD across three runs; the dashed line and band show the TextMAS mean $\pm$ SD from Table~\ref{tab:main-results}.}
\label{fig:length-selection}
\end{figure}

Among the evaluated communication lengths, 400 positions gives the highest mean accuracy: 73.33\% with a standard deviation of 1.67 percentage points (Figure~\ref{fig:length-selection}).
The 40-position result is 70.56\%, giving an increase of approximately 2.8 percentage points at 400 positions.
Increasing the communication length to 500 positions yields 72.78\%.
The GAF+$H+J26$ result on AIME2024+2025 in Table~\ref{tab:main-results}, the J26 point in Figure~\ref{fig:j-layer-ablation}, and the $M=400$ point in Figure~\ref{fig:length-selection} use the same Sender cache and evaluation configuration.

\subsection{External Two-Agent Baselines}

\begin{table*}[t]
\centering\small\setlength{\tabcolsep}{7pt}
\resizebox{\textwidth}{!}{\begin{tabular}{@{}lrrrrrr@{}}
\toprule
& \multicolumn{3}{c}{Qwen3-4B $\to$ 4B} & \multicolumn{3}{c}{Qwen3-8B $\to$ 8B} \\
\cmidrule(lr){2-4}\cmidrule(l){5-7}
Method & MATH500 & GPQA-Diamond & AIME24+25 & MATH500 & GPQA-Diamond & AIME24+25 \\
\midrule \\
GAF+$H+J26$ & \textbf{92.00} & 41.92 & \textbf{68.33} & \textbf{95.20} & 49.49 & \textbf{75.00} \\
LatentMAS   & 83.20          & 38.38 & 61.67          & 86.80          & 47.10 & 51.67 \\
StateBridge & 89.40          & \textbf{45.45} & 61.67 & 93.20          & \textbf{52.50} & 68.33 \\
CIPHER      & 80.40          & 35.35 & 58.33          & 84.20          & 41.41 & 61.67 \\
\bottomrule
\end{tabular}
}
\caption{Homogeneous two-agent communication accuracy (\%), averaged over three runs. GAF+$H+J26$ uses 400 communicated positions.}
\label{tab:external}
\end{table*}

We compare J-Comm with LatentMAS, StateBridge, and CIPHER under the same Receiver backend and prompts (Table~\ref{tab:external}).
LatentMAS uses its official latent recurrence and realignment; replaying the consumed Sender inputs in vLLM reconstructs its working-memory KV cache before the Receiver prompt.
StateBridge uses a vLLM text plan, HF teacher-forced state extraction, and its official mapping with up to 64 post-thinking states.
CIPHER communicates probability-weighted token embeddings.

With two Qwen3-4B agents, J-Comm achieves 92.00\% on MATH500, compared with 83.20\% for LatentMAS, 89.40\% for StateBridge, and 80.40\% for CIPHER.
On AIME, J-Comm reaches 68.33\%, compared with 61.67\%, 61.67\%, and 58.33\%, respectively.
On GPQA, J-Comm obtains 41.92\%, above LatentMAS (38.38\%) and CIPHER (35.35\%) and below StateBridge (45.45\%).

With two Qwen3-8B agents, J-Comm achieves 95.20\% on MATH500, compared with 86.80\% for LatentMAS, 93.20\% for StateBridge, and 84.20\% for CIPHER.
Its AIME accuracy is 75.00\%, exceeding the three baselines at 51.67\%, 68.33\%, and 61.67\%.
On GPQA, J-Comm reaches 49.49\%, compared with 47.10\% for LatentMAS, 52.50\% for StateBridge, and 41.41\% for CIPHER.
J-Comm leads the homogeneous MATH500 and AIME comparisons at both model scales, while StateBridge leads GPQA.

\subsection{Cross-Family Communication}

We transfer Qwen3-8B messages to Llama-3.2-3B-Instruct using token-string correspondence and native Receiver embeddings and Jacobian anchors (Appendix~\ref{app:cross-family}).
On GPQA-Diamond, $H+J$ obtains a mean of 75 correct answers out of 198, compared with 71 for $H$, 73 for TextMAS, 69 for shuffled $J$, and 40 for Receiver-only (Table~\ref{tab:cross-family}).
The matched $J$ message improves on $H$ by four mean correct answers and on shuffled $J$ by six, extending the interface across both model families and vocabularies.

\section{Related Work}
\label{sec:related}

\paragraph{Continuous reasoning.}
Coconut trains a recurrence that feeds hidden states back as inputs \citep{hao2025coconut}, while Soft Thinking uses probability-weighted token embeddings in a frozen model \citep{zhang2025softthinking}.
\galr{} incorporates candidate geometry by multiplying predictive probabilities by their probability-weighted geometric support before constructing the next input.

\paragraph{Latent communication.}
CIPHER communicates probability-weighted embeddings through the Receiver's embedding table \citep{pham2024cipher}.
Interlat learns a communication adapter \citep{du2025interlat}, and ThoughtComm learns shared and private latent representations with a prefix adapter \citep{zheng2025thoughtcomm}.
LatentMAS aligns hidden states for latent generation and transfers layer-wise KV caches \citep{zou2025latentmas}; StateBridge uses closed-form Procrustes alignment for same-model communication \citep{peng2026statebridge}.
\jcomm{} pairs distribution-derived inputs with intermediate-state readouts, using fixed model-specific dictionaries while keeping KV caches local.

\paragraph{Activation and Jacobian interfaces.}
\citet{ramesh2025activations} exchange intermediate activations through addition, averaging, or replacement.
The Jacobian lens supplies vocabulary-indexed intermediate-state readouts and coordinate interventions \citep{gurnee2026verbalizable}.
Our $J$ channel transmits token IDs and signed readouts, reconstructs updates through the Receiver's regularized dual dictionary, and writes them at the paired $H$ positions during prefill.

\paragraph{Geometric reweighting.}
GAF is a Baum--Eagon growth transformation and a symmetric replicator step \citep{baum1967inequality,losert1983dynamics}, related to affinity updates in dominant-set clustering \citep{pavan2007dominant}.
\galr{} applies one such update to each model-generated candidate distribution.

\section{Conclusion}
\label{sec:conclusion}

We introduced \jcomm{}, combining geometric agreement reasoning with dual-channel latent communication between frozen LLMs.
\gear{} constructs continuous inputs through geometric reweighting, while $H+J$ transfers paired feedback distributions and intermediate-state readouts across model widths.
In the primary Qwen3-8B $\to$ Qwen3-4B evaluation, \jcomm{} achieves the highest mean score across seven benchmarks while reducing generated text tokens relative to TextMAS.

\section*{Limitations}

J-Comm uses a token correspondence and calibrated Jacobian dictionaries for each model pair.
Cross-vocabulary communication adds token-string matching and subtoken decomposition.
The adaptive Sender can consume up to 32,768 latent steps, so its computation remains a heavy component of total inference cost even with a short communicated suffix.

\clearpage
\bibliography{references}

\appendix
\raggedbottom
\renewcommand{\dbltopfraction}{0.9}
\section{Current Configuration}
\label{app:settings}

\begin{table}[!htbp]
\centering
\small
\begin{tabular}{@{}ll@{}}
\toprule
Setting & Value \\
\midrule
Sender / Receiver & Qwen3-8B / Qwen3-4B \\
Sender / Receiver block & 26 / 26 \\
Sender temperature & 0.6 \\
Initial top-$k$ candidates & 30 \\
Nucleus threshold & 0.95 \\
Relative min-$p$ threshold & 0.001 \\
Maximum feedback candidates & 10 \\
Minimum / maximum steps & 40 / 32,768 \\
Cold entropy threshold & 0.01 nats \\
Consecutive cold steps & 256 \\
Communication length $M$ & 400 \\
$J$ anchors per position $P$ & 32 \\
Dual ridge $\lambda$ / gain $\alpha$ & $10^{-3}$ / 1 \\
Magnitude reference block & 29 \\
Receiver output cap & Task-specific (Table~\ref{tab:text-budgets}) \\
Receiver temperature / top-$p$ & 0.6 / 0.95 \\
Receiver top-$k$ & 20 \\
\bottomrule
\end{tabular}
\caption{Default Qwen3-8B $\to$ Qwen3-4B configuration. Sender latent steps, communication positions, and Receiver output tokens have separate budgets.}
\label{tab:settings}
\end{table}

\subsection{Implementation Details}
\label{sec:implementation}

\paragraph{Layers and calibration.}
Our main configuration reads Sender block 26 and writes Receiver block 26, using zero-based implementation indices.
The layer-26 communication setting of \citet{ramesh2025activations} motivated this candidate.
Figure~\ref{fig:j-layer-ablation} compares Receiver write blocks 20, 26, and 29.
We calibrate the $J$ lenses on generic text and keep both model weights and lens dictionaries fixed during inference.
The Receiver-26 lens uses 32 calibration prompts.

\paragraph{Candidate filtering.}
We divide logits by temperature, retain the top 30 logits, and normalize their probabilities.
Nucleus filtering retains candidates through the first cumulative-probability crossing of 0.95, followed by renormalization.
We remove candidates whose probability is below $0.001$ times the largest probability, retain at most the first 10 candidates, and renormalize to obtain $p_t$.
All geometric calculations use these candidates.

\paragraph{Stopping.}
After 40 consumed steps, a top-$q$ end-of-thinking candidate terminates the trajectory before that candidate is consumed.
The cold-stop statistic is the entropy of the full-vocabulary temperature-scaled predictive distribution, before candidate filtering or geometric reweighting.
An entropy below 0.01 nats for 256 consecutive consumed steps triggers termination after the current step.
The hard cap is 32,768 consumed steps.

\paragraph{Numerics and inference.}
Candidate probabilities, cosine kernels, and dual reconstruction are computed in float32; latent inputs and model execution use bfloat16.
The primary Receiver uses vLLM 0.10.2 V0, with prefix caching and chunked prefill disabled for the $H/J$ hooks, a native context limit of 40,960, and no YaRN extension.
The $H$ radius is calibrated separately from $J$'s relative intervention magnitude.

\subsection{Text Output Budgets}
\label{app:task-budgets}

Table~\ref{tab:text-budgets} lists the maximum number of output tokens for each agent on each dataset.
ALFWorld allows at most 50 action turns per episode, with each agent capped at 2,048 output tokens per turn.

\begin{table}[!htbp]
\centering\small
\setlength{\tabcolsep}{5pt}
\begin{tabular}{@{}lrr@{}}
\toprule
Dataset & \shortstack{Sender\\token cap} & \shortstack{Receiver\\token cap} \\
\midrule
GSM8K & 1,024 & 2,048 \\
MATH500 & 4,096 & 8,192 \\
GPQA-Diamond & 4,096 & 8,192 \\
MedQA & 2,048 & 4,096 \\
MBPP+ & 2,048 & 4,096 \\
AIME2024+2025 & 8,192 & 31,000 \\
ALFWorld & 2,048/turn & 2,048/turn \\
\bottomrule
\end{tabular}
\caption{Maximum text output tokens per agent. ALFWorld budgets apply to each action turn, with at most 50 turns per episode.}
\label{tab:text-budgets}
\end{table}

\section{Additional Experiments}

\subsection{Cross-Family Communication}
\label{app:cross-family}

\begin{table}[!htbp]
\centering\small
\begin{tabular}{@{}lrrrrr@{}}
\toprule
Task & R & Text & $H$ & $HJ$ & $HJ_{\rm shuf}$ \\
\midrule
GPQA & 40 & 73 & 71 & 75 & 69 \\
\bottomrule
\end{tabular}

\caption{Qwen3-8B $\to$ Llama-3.2-3B-Instruct on GPQA-Diamond: mean correct answers over three runs, out of 198 questions. Each condition uses 40 communicated positions when applicable and an 8,192-token Receiver cap; the TextMAS Sender cap is 4,096. $HJ$ writes at native Llama block 20.}
\label{tab:cross-family}
\end{table}

We preserve the Sender's consumed $q$ distribution while changing the vocabulary bridge.
Exact shared byte-level token strings use the corresponding Llama embedding; an unmatched ordinary token is decomposed using Llama's byte-level BPE and represented by the mean of its subtoken embeddings.
The $q$-weighted mixture is then normalized to the native Receiver embedding scale, still producing one continuous position per Sender step.
The $J$ channel recomputes signed top-32 readouts over 109,566 paired token anchors and reconstructs through the previously calibrated native Llama lens (write block 20, magnitude reference block 22).
The bridge reuses the Sender trajectory and calibrated Receiver lens.
Matched $HJ$ obtains 75 mean correct answers, compared with 71 for $H$, 69 for shuffled $J$, and 73 for TextMAS (Table~\ref{tab:cross-family}).

\subsection{Case Study: Inspecting Task-Related $J$ Coordinates}
\label{app:j-case-study}

We inspect a MedQA mechanism question and a MATH500 counting question in the Qwen3-8B $\to$ 4B setting (Figure~\ref{fig:j-case-study}).
Both examples use 400 communicated positions, 32 anchors per position, and Receiver block 26.
The selected vocabulary coordinates expose task-related concepts in the transmitted messages.

\begin{figure*}[!t]
  \centering
  \includegraphics[width=\textwidth]{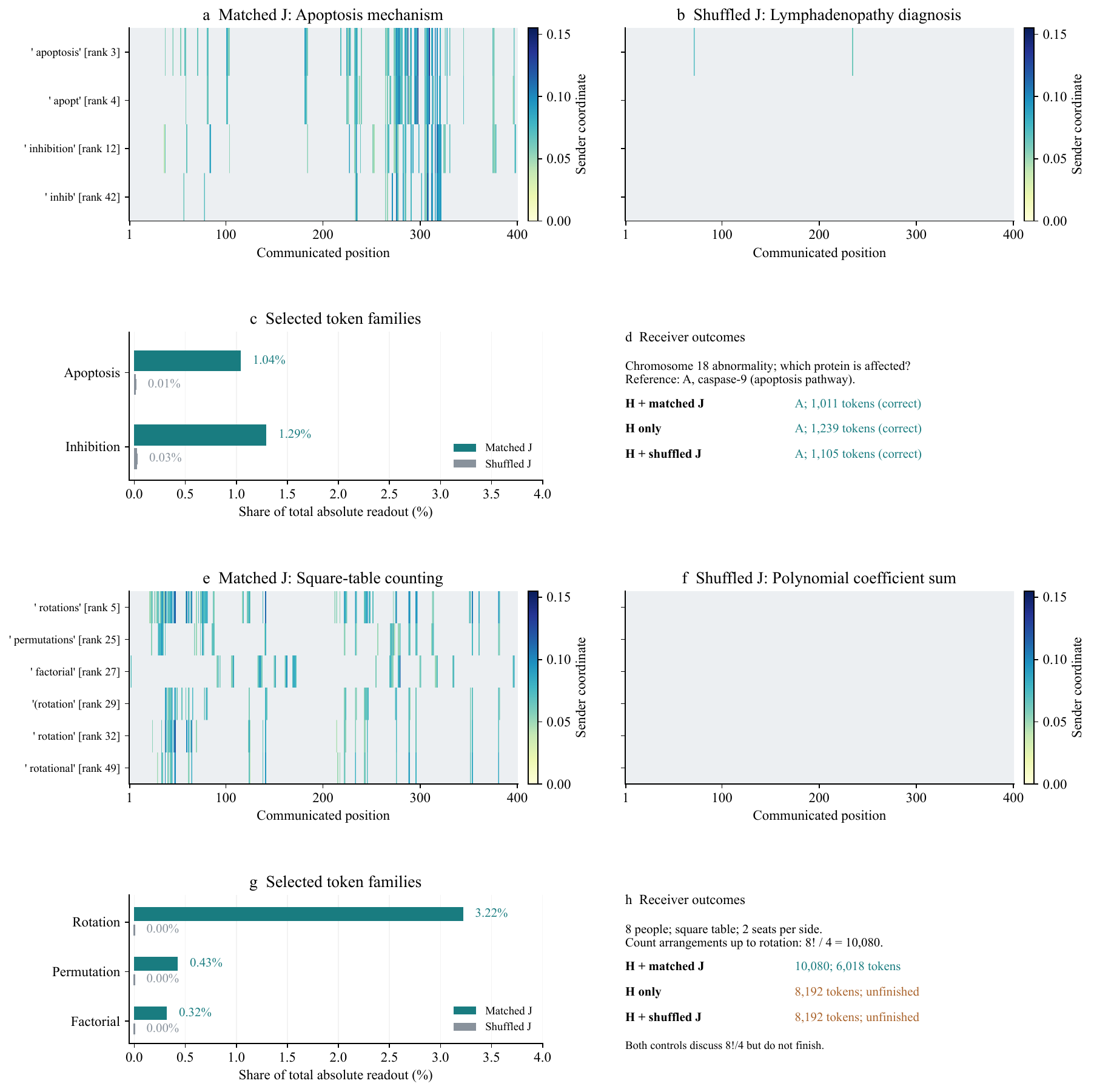}
  \caption{Inspectable $J$ coordinates and shuffled donors for a MedQA apoptosis question (a--d) and MATH500 square-table counting (e--h), both with 400 communication positions. Gray cells denote absent sparse coordinates. All three medical-case conditions answer correctly; only matched $J$ completes the counting solution within the output budget.}
  \label{fig:j-case-study}
\end{figure*}

\paragraph{MedQA: Apoptosis Regulation.}
The MedQA question describes a chromosome-18 abnormality and asks which protein's regulation is most affected; the reference answer is caspase-9 (option A).
The relevant intermediate mechanism is apoptosis inhibition associated with BCL2, which is also discussed in the Receiver explanation.
The matched message contains \texttt{apoptosis}, \texttt{apopt}, and \texttt{inhibition}, ranked third, fourth, and twelfth by total absolute readout magnitude.
Neither apoptosis nor inhibition is explicitly named in the question or its answer choices.
The \texttt{apoptosis} coordinate occurs at 68 of the 400 positions, compared with two positions in the shuffled donor, a different lymphadenopathy question.
Tokens containing \texttt{apopt} account for 1.04\% of the matched message's absolute readout magnitude, versus 0.0145\% for the donor.
The heatmaps show the selected sparse readouts, and the bars summarize case-insensitive \texttt{apopt} and \texttt{inhib} families.

All three conditions answer A correctly: matched $H+J$ uses 1,011 generated tokens, $H$ alone uses 1,239, and shuffled $H+J$ uses 1,105, under a 4,096-token cap.
Their explanations all identify the BCL2/apoptosis mechanism.
The shuffled message comes from a different question under a random derangement of the MedQA set.

\paragraph{MATH500: Counting under Rotational Symmetry.}
Figure~\ref{fig:j-case-study} (bottom) provides a complementary MATH500 example: eight people sit around a square table with two seats per side, and arrangements equivalent under rotation are identified.
The correct count is $8!/4=10{,}080$.
This example uses 400 communicated positions, 32 anchors per position, Receiver block 26, and an 8,192-token output cap in the same Qwen3-8B $\to$ 4B setting.
The matched message contains \texttt{permutations}, \texttt{factorial}, and \texttt{rotations}, connecting inspectable vocabulary coordinates to counting operations in the solution.
Case-insensitive token families containing \texttt{permut}, \texttt{factorial}, and \texttt{rotat} account for 0.43\%, 0.32\%, and 3.22\% of total absolute readout magnitude, respectively; none appears in the shuffled donor's sparse message, which comes from a polynomial coefficient-sum question.

With the same recipient $H$ and output budget, matched $H+J$ completes correctly in 6,018 tokens, whereas $H$ alone and shuffled $H+J$ reach the 8,192-token cap without completing an answer.
Both controls discuss $8!/4$ alongside the competing $7!$ heuristic.

\section{Sender and Receiver Prompts}
\label{app:prompts}
\begingroup
\renewenvironment{quote}
  {\list{}{\leftmargin=1em\rightmargin=1em
    \topsep=3pt\parsep=2pt\partopsep=0pt}\item\relax}
  {\endlist}
We reproduce the user-message templates used in our Qwen experiments. Braced placeholders denote task-dependent content, rather than literal prompt text. Each agent applies its own tokenizer's chat template with \texttt{add\_generation\_prompt=True} and \texttt{enable\_thinking=True}; the implementation then appends \texttt{<think>} followed by a newline. The role instructions below are supplied as a user message.

\subsection{Mathematical Reasoning}
The GSM8K, MATH500 and AIME experiments use the following templates.
\paragraph{Sender (Planner).}

\begin{quote}
\small\ttfamily\raggedright
You are the Planner Agent. Develop a concise step-by-step plan and useful intermediate deductions for solving the following question. The Solver Agent will see the same question and continue your work. Do not give the final answer.
\par

\par
Question: \{QUESTION\}
\end{quote}

\begin{samepage}
\paragraph{Receiver (Solver).}

\begin{quote}
\small\ttfamily\raggedright
You are the Solver Agent. Solve the question below step by step and put your final answer inside \textbackslash{}boxed\{\}. You may receive reference work from a Planner Agent. Check it and continue any useful reasoning; ignore it if it is irrelevant or incorrect. The reference may be empty.
\par

\par
Question: \{QUESTION\}
\par

\par
Planner's reference work:
\par
\{MESSAGE\}
\end{quote}
\end{samepage}

For latent communication, \texttt{\{MESSAGE\}} is replaced by the $M$ continuous $H$ input positions. The $H+J$ condition uses the same prompt and positions, with $J$ applied at the intermediate layer during prefill; $J$ is not rendered as text. TextMAS inserts the Sender's textual message at the same location. The Receiver-only condition leaves this location empty.

\subsection{Multiple-Choice and Code Tasks}
\paragraph{Sender.}

\begin{quote}
\small\ttfamily\raggedright
You are the Planner Agent. Develop a concise plan and useful intermediate deductions for the following question. The Solver Agent sees the same question. Do not give the final answer.
\par

\par
\{QUESTION\}
\end{quote}

\paragraph{Receiver.}

\begin{quote}
\small\ttfamily\raggedright
You are the Solver Agent. Solve the following question. \{ANSWER\_INSTRUCTION\} You may receive a Planner reference; check it and continue useful reasoning. The reference may be empty.
\par

\par
Question: \{QUESTION\}
\par

\par
Planner reference:
\par
\{MESSAGE\}
\end{quote}

For GPQA-Diamond and MedQA, the question includes the answer options and \texttt{\{ANSWER\_INSTRUCTION\}} is:

\begin{quote}
\small\ttfamily\raggedright
Choose one option. Put the single final option letter inside \textbackslash{}boxed\{\}.
\end{quote}

For MBPP+, the answer instruction is:

\begin{quote}
\small\ttfamily\raggedright
Return a complete Python implementation in a python code block.
\end{quote}

The MBPP+ question itself is prefixed with:

\begin{quote}
\small\ttfamily\raggedright
Write a self-contained Python implementation for this problem. Return the complete function and any necessary imports in a python code block.
\par

\par
\{PROBLEM\}
\end{quote}

\subsection{Compact-Plan TextMAS Variant}
The compact-plan variant appends the following instruction to the Sender's user message, where \texttt{\{PLAN\_CAP\}} is the configured textual-plan budget. This addition is specific to that variant, rather than the original full-output TextMAS prompt.

\begin{quote}
\small\ttfamily\raggedright
After thinking, send the Solver only a concise plan of 3-6 numbered steps, within \{PLAN\_CAP\} tokens. Include useful intermediate deductions, key checks and unresolved uncertainties. Do not repeat the question or give the final answer.
\end{quote}

For MBPP+, it additionally appends:

\begin{quote}
\small\ttfamily\raggedright
Describe the algorithm, edge cases and complexity; do not write the full implementation.
\end{quote}

\subsection{ALFWorld}
Both agents receive the following shared rules, followed by the interaction history.

\begin{quote}
\small\ttfamily\raggedright
Interact with the household to complete the stated task. Choose exactly one next action using the observation and previous actions. Objects and receptacles must use their exact names and numbers.
\par
\noindent\begin{minipage}{\linewidth}
\setlength{\parskip}{2pt}
Available action forms:
\par
go to \{receptacle\}
\par
open \{receptacle\}
\par
close \{receptacle\}
\par
take \{object\} from \{receptacle\}
\par
move \{object\} to \{receptacle\}
\par
use \{object\}
\par
examine \{object or receptacle\}
\par
clean \{object\} with \{receptacle\}
\par
heat \{object\} with \{receptacle\}
\par
cool \{object\} with \{receptacle\}
\par
look
\par
inventory
\end{minipage}
\par
Navigate to a receptacle before interacting with it. Closed containers must be opened. "Nothing happens" means the action did not change the environment.
\par
Think about the next useful action. End your response with exactly one line: Action: \textless{}action\textgreater{}.
\end{quote}

The Sender's message is \texttt{\{RULES\}}, two newlines, and \texttt{\{HISTORY\}}. The history starts with \texttt{Observation:} and the initial observation, then appends each executed \texttt{Action:} and resulting \texttt{Observation:}.
\paragraph{Receiver with latent communication.}

\begin{quote}
\small\ttfamily\raggedright
\{RULES\}
\par

\par
\{HISTORY\}
\par

\par
Planner latent message:
\par
\{MESSAGE\}
\par

\par
Choose the next action.
\end{quote}

TextMAS uses the same template with \texttt{Planner text message:} in place of \texttt{Planner latent message:}, and inserts the generated text. The independent Receiver receives only the shared rules and history.
\endgroup

\section{GAF: Proof and Supporting Derivations}
\label{app:gaf}

\subsection{Geometric Selectivity}

The mass assigned to a given target depends on how strongly the other candidates agree with one another.
The following result makes this dependence explicit.

Fix a target candidate with unit direction $u$ and prior mass $g\in(0,1)$.
Write $v_i$ for at most $n\ge2$ remaining directions and $w_i$ for their conditional weights, so $w_i\ge0$ and $\sum_iw_i=1$.
Define their mean affinity to the target and their internal agreement by
\begin{equation}
 \chi=\sum_iw_i e^{u^\top v_i-1},\qquad
 \mathcal B=\sum_{i,j}w_iw_j e^{v_i^\top v_j-1}.
 \label{eq:gaf-affinities}
\end{equation}
We consider the regime where $\chi$ is at most $e^{-1}$, the kernel value for orthogonal directions.
This is a condition on the weighted mean affinity; individual candidates may have larger affinities to the target.

\begin{theorem}[Geometric selectivity of one GAF step]
\label{thm:gaf}
Suppose $e^{-2}\le \chi\le e^{-1}$, and set
\begin{align}
 \tau&=1+\log \chi,\nonumber\\
 b_n(\chi)&=\frac1n+\frac{n-1}{n}
   e^{-\frac{n}{n-1}(1-\tau^2)},\nonumber\\
 Q(g,\chi,b)&=\frac{g[g+(1-g)\chi]}
 {g^2+2g(1-g)\chi+(1-g)^2b}.
 \label{eq:gaf-envelope}
\end{align}
Then $b_n(\chi)\le \mathcal B\le1$, and the updated target mass satisfies
\begin{equation}
 Q(g,\chi,1)\le q_\star\le Q(g,\chi,b_n(\chi)).
 \label{eq:gaf-mass-range}
\end{equation}
Both endpoints are attainable over admissible geometries when $d_S\ge n$.
Let $\beta_n(\chi)=[b_n(\chi)-\chi]/[1+b_n(\chi)-2\chi]$.
Every admissible geometry decreases target mass if $g<\beta_n(\chi)$ and increases it if $g>1/2$.
For $\beta_n(\chi)<g<1/2$ and $d_S\ge n$, geometries with gain and with loss both exist at the same $(g,\chi,n)$.
At $g=\beta_n(\chi)$ the gain is nonpositive; at $g=1/2$ it is nonnegative, and these zero bounds are attained when $d_S\ge n$.
At $\chi=e^{-2}$ the thresholds coincide at $1/2$.
\end{theorem}

The proof below uses realizable spherical configurations: an equal-weight latitude simplex attains the lower energy bound, and coincident opponent directions attain the upper bound.
The theorem characterizes the target mass after one reweighting in the stated affinity regime. A target group obeys the same formulas when its directions coincide.

\subsection{Proof of Theorem~\ref{thm:gaf}}
\begin{proof}
Let $\mu=\sum_iw_iv_i$ and $r=\sum_iw_i^2\ge1/n$.
Jensen's inequality and $\tau\le0$ give
\[
 e^{u^\top\mu-1}\le \chi,\qquad
 u^\top\mu\le \tau\le0,\qquad \norm{\mu}^2\ge \tau^2.
\]
For $r<1$, the ordered off-diagonal weights $w_iw_j/(1-r)$ sum to one.
Since $\sum_{i\ne j}w_iw_jv_i^\top v_j=\norm{\mu}^2-r$, another application of Jensen yields
\begin{align*}
 \mathcal B&\ge r+(1-r)\exp\!\left(\frac{\norm{\mu}^2-1}{1-r}\right)\\
  &\ge r+(1-r)\exp\!\left(-\frac{1-\tau^2}{1-r}\right).
\end{align*}
For $a=1-\tau^2\ge0$, the last expression increases with $r$: its derivative is $1-(1+z)e^{-z}\ge0$, where $z=a/(1-r)$.
Substitution of $r\ge1/n$ proves $\mathcal B\ge b_n(\chi)$.
If $r=1$, then $\mathcal B=1$ and the same bound holds.
The upper bound follows from $v_i^\top v_j\le1$.

For attainability when $d_S\ge n$, choose unit vectors $z_1,\ldots,z_n$ forming a regular simplex in $u^\perp$, so $z_i^\top z_j=-1/(n-1)$ for $i\ne j$, and set
\[
 v_i=\tau u+\sqrt{1-\tau^2}\,z_i,\qquad w_i=1/n.
\]
All target affinities equal $e^{\tau-1}=\chi$, and all off-diagonal opponent affinities equal $c=\exp[-n(1-\tau^2)/(n-1)]$.
Thus $\mathcal B=1/n+(1-1/n)c=b_n(\chi)$.
A single opponent at any unit direction $v$ satisfying $u^\top v=\tau$ attains $\mathcal B=1$ at the same $\chi$.
Multiple candidate indices may also share this direction.
In fact every intermediate energy is attainable: retaining the latitude simplex, set
\begin{align*}
 w_1&=\frac{1+(n-1)a}{n},\\
 w_2=\cdots=w_n&=\frac{1-a}{n},\qquad a\in[0,1].
\end{align*}
The cross-affinity stays fixed and
\[
 \mathcal B=c+(1-c)\frac{1+(n-1)a^2}{n}
\]
ranges continuously from $b_n(\chi)$ to $1$.
If exactly $n\ge2$ distinct positively weighted directions are required and $\chi>e^{-2}$, the upper endpoint is only a supremum, approached as $a\to1$.
At $\chi=e^{-2}$, every supported opponent is $-u$, so $\mathcal B=1$.

The target support is $g+(1-g)\chi$ and the total agreement is
\[
 C=g^2+2g(1-g)\chi+(1-g)^2\mathcal B>0.
\]
Consequently $q_\star=Q(g,\chi,\mathcal B)$ and
\begin{equation}
 q_\star-g
 =\frac{g(1-g)\{g(1-\chi)-(1-g)(\mathcal B-\chi)\}}{C}.
 \label{eq:gaf-exact-gain}
\end{equation}
The updated mass decreases strictly with $\mathcal B$, proving Equation~\ref{eq:gaf-mass-range} and its attainability.
For $\tau\in(-1,0]$, strict Jensen applied to the diagonal/off-diagonal mixture defining $b_n$ gives
\[
 b_n(\chi)>e^{\tau^2-1}\ge e^{\tau-1}=\chi.
\]
At $\tau=-1$, $b_n(\chi)=1>\chi$ as well.
The zero-gain threshold
\[
 H(\mathcal B,\chi)=\frac{\mathcal B-\chi}{1+\mathcal B-2\chi}
\]
has derivative $\partial_{\mathcal B}H=(1-\chi)/(1+\mathcal B-2\chi)^2>0$.
It therefore ranges from $\beta_n(\chi)$ to $1/2$ as $\mathcal B$ ranges from $b_n(\chi)$ to $1$.
Equation~\ref{eq:gaf-exact-gain} and the endpoint constructions prove the stated gain and loss cases, including their boundaries.
\end{proof}

The argument uses classical spherical-energy ideas.
Regular simplices are universally optimal for completely monotonic squared-distance potentials, including $e^{-\norm{v_i-v_j}^2/2}$ \citep{cohn2007universally}; weighted spherical potentials and constrained energy minimization are studied by \citet{borodachov2026bounds}.
Here the energy bound is conditioned on the target affinity $\chi$ and used to characterize one GAF step.

\subsection{Agreement Gain and Consumed-Input Displacement}
For a fixed candidate set, let $s=Kp$, $C=C(p)>0$, $\xi=q-p$, and
\[
 V_s=\sum_i p_i(s_i-C)^2.
\]
The identity $\xi_i=p_i(s_i-C)/C$ implies $\xi^\top Kp=V_s/C$, so expansion of $C(p+\xi)$ gives
\begin{equation}
 C(q)-C(p)=\frac{2V_s}{C}+\xi^\top K\xi.
 \label{eq:gaf-agreement-gain}
\end{equation}
The exponential kernel is positive semidefinite because
\[
 K_{ij}=e^{-1}\sum_{k=0}^{\infty}
 \frac{\langle u_i^{\otimes k},u_j^{\otimes k}\rangle}{k!}
\]
is an absolutely convergent sum of positive-semidefinite kernels.
Thus both terms in Equation~\ref{eq:gaf-agreement-gain} are nonnegative.
Equality holds exactly when $V_s=0$, equivalently $q=p$, or $s_i=C$ on the support of $p$.
This is also the Baum--Eagon inequality applied to the homogeneous polynomial $C(p)$: since $\partial_iC=2(Kp)_i$, its growth transformation is precisely Equation~\ref{eq:q} \citep{baum1967inequality}.
The same multiplicative form appears in symmetric replicator dynamics \citep{hofbauer1998evolutionary,losert1983dynamics} and affinity-based dominant-set methods \citep{pavan2007dominant}.

For the raw, nonzero embeddings used by Equation~\ref{eq:feedback}, define
\begin{gather*}
 X_p=\sum_i p_iE_i^S,\qquad X_q=\sum_iq_iE_i^S,\\
 V_E=\sum_i p_i\norm{E_i^S-X_p}^2.
\end{gather*}
Since $\sum_i p_i(s_i-C)=0$,
\begin{align*}
 X_q-X_p&=\frac1C\sum_i p_i(s_i-C)(E_i^S-X_p),\\
 \norm{X_q-X_p}^2
 &\le\frac{V_EV_s}{C^2}
 \le\frac{V_E}{2C}\{C(q)-C(p)\}.
\end{align*}
The first inequality is weighted Cauchy--Schwarz.
These identities hold for every candidate distribution under the fixed kernel.
An unchanged mixture need not imply unchanged weights; that converse requires affine independence of the supported raw embeddings.

\subsection{Boundary Cases and Approximate Coherence}
The energy bound remains valid for $d_S<n$, but its lower endpoint need not be attainable.
With one supported opponent, $\mathcal B=1$ and the exact gain threshold is $1/2$; the formula for $b_n$ is not used at $n=1$.
The upper limit $\chi\le e^{-1}$ is essential: two equal-weight opponents at $u$ and $-u$ give $\chi=\mathcal B=(1+e^{-2})/2\approx0.56767$, whereas extending $b_2$ to this value would give $b_2(\chi)\approx0.59859>\mathcal B$.

For a target group of total mass $g$, suppose its unit directions $u_i$ satisfy $\norm{u_i-u}\le\delta$.
Let $A$ denote its conditional within-group affinity and $X$ its actual cross-affinity, retaining $\chi$ for the representative $u$.
Pairwise target distances are at most $2\delta$, and $s\mapsto e^{s-1}$ is $1$-Lipschitz on $[-1,1]$, giving
\[
 A\ge e^{-2\delta^2},\qquad |X-\chi|\le\delta.
\]
Set $L=g(1-\chi)-(1-g)(\mathcal B-\chi)$.
The actual difference between target and opponent mean support, $D=g(A-X)-(1-g)(\mathcal B-X)$, satisfies
\begin{align*}
 L-g(1-e^{-2\delta^2})-|1-2g|\delta
 &\le D\\
 &\le L+|1-2g|\delta.
\end{align*}
This follows by subtracting $L$ from $D$, leaving $g(A-1)+(1-2g)(X-\chi)$.
The group gains mass exactly when $D>0$ and loses mass exactly when $D<0$.
Together with $b_n(\chi)\le \mathcal B\le1$, these bounds give sufficient gain and loss conditions for approximately coherent groups in the stated affinity regime.

\section{Proof and Derivations for Finite-Write Fidelity}
\label{app:j-fidelity}

\subsection{Effect of a Finite J Write}
\label{sec:j-fidelity}

The reconstructed write passes through the Receiver's remaining layers and final normalization.
We analyze its effect relative to the $H$-only baseline, with the prompt, $H$ inputs, and output prefix fixed.
The token-ordering condition below compares the actual output with the average lens's prediction of the output after the write.
Abbreviate $B_t=B^R_{\ell_R,I_t}$, $G_t=B_tB_t^\top$, $d_t=d^R_{\ell_R,t}$, $r_t=r^R_{\ell_R,t}$, and $\mathcal J_R=\mathcal J^R_{\ell_R}$.
Assume every selected raw lens row and every written $d_t$ is nonzero.
The actual update at each retained position is
\[
 \delta_t=\kappa_t d_t,\qquad
 \kappa_t=\frac{\alpha\norm{r_t}\rho_t}{\norm{d_t}}.
\]
Let $r$ and $\delta$ stack all $M$ retained positions. Let $F(r)$ map these residuals to the pre-final-RMSNorm residual at the fixed output position; unedited positions are fixed parameters.
Write $y=F(r)$, $y'=F(r+\delta)$, and
\begin{gather*}
 w_v=\operatorname{diag}(\gamma_R)u^R_v,\qquad
 s(z)=\sqrt{\norm{z}^2/d_R+\varepsilon},\\
 \ell_v(z)=w_v^\top z/s(z),
\end{gather*}
for a bias-free output head with $\varepsilon>0$.
Define the nominal endpoint, its finite mismatch, and the nominal response by
\begin{align*}
 \widehat y&=y+\sum_t\mathcal J_R\delta_t,
 &e&=y'-\widehat y,\\
 \zeta_v&=w_v^\top(\widehat y-y),
 &n_v&=w_v^\top\widehat y.
\end{align*}
Thus $\zeta_v$ retains the norms of the raw lens rows $(\mathcal J_R)^\top w_v$, which unit anchors alone discard.

\begin{theorem}[Finite-write fidelity and conditional output ordering]
\label{thm:j-fidelity}
For $\lambda>0$, put $A_t=G_t(G_t+\lambda I)^{-1}$.
The selected-anchor readback satisfies
\begin{align}
 B_t(r_t+\delta_t)&=B_tr_t+\kappa_tA_tc_t,
 \label{eq:j-fidelity-readback}\\
 c_t-A_tc_t&=\lambda(G_t+\lambda I)^{-1}c_t.
 \notag
\end{align}
Writing $s=s(y)$ and $s'=s(y')$, every vocabulary token has the exact finite logit response
\begin{equation}
 \ell_v(y')-\ell_v(y)
 =\frac{\zeta_v+w_v^\top e-(s'-s)\ell_v(y)}{s'}.
 \label{eq:j-fidelity-logit}
\end{equation}
Let a finite $V\subseteq\mathcal V$ contain at least two tokens, and let $v$ uniquely maximize $n_j$ over $j\in V$.
Define
\begin{equation}
 R_v=\min_{j\in V\setminus\{v\}}
 \frac{n_v-n_j}{\norm{w_v-w_j}}.
 \label{eq:j-fidelity-radius}
\end{equation}
For $E\ge0$, $v$ remains the unique highest-logit token in $V$ for every endpoint error satisfying $\norm{e}\le E$ if and only if $E<R_v$.
\end{theorem}

The identities include ridge shrinkage and final normalization at the implemented gain $\alpha=1$.
The condition $E<R_v$ preserves the lens-predicted winner at the fixed prefix; setting $V=\mathcal V$ gives full-vocabulary ordering.

\begin{proof}[Proof of Theorem~\ref{thm:j-fidelity}]
The reconstruction in Equations~\ref{eq:dual}--\ref{eq:reconstruct} is
$d_t=B_t^\top(G_t+\lambda I)^{-1}c_t$.
Multiplying by $B_t$ gives $B_td_t=A_tc_t$, and substituting
$\delta_t=\kappa_td_t$ proves Equation~\ref{eq:j-fidelity-readback}.
The ridge-bias identity follows from
$I-A_t=\lambda(G_t+\lambda I)^{-1}$.
$G_t+\lambda I$ is positive definite for every $\lambda>0$.

For the functional identity, the definition of $e$ gives
$y'=y+\sum_t\mathcal J_R\delta_t+e$.
Therefore $w_v^\top y'=w_v^\top y+\zeta_v+w_v^\top e$.
Using the actual final RMSNorm denominators,
\begin{align*}
 \ell_v(y')-\ell_v(y)
 &=\frac{w_v^\top y+\zeta_v+w_v^\top e}{s'}
   -\frac{w_v^\top y}{s}\\
 &=\frac{\zeta_v+w_v^\top e-(s'-s)\ell_v(y)}{s'}.
\end{align*}

For output ordering, RMSNorm has a common positive denominator for all tokens.
For every $j\in V\setminus\{v\}$, the sign of the actual logit difference is thus the sign of
\[
 n_v-n_j+(w_v-w_j)^\top e.
\]
Its exact minimum over $\norm e\le E$ is
$n_v-n_j-E\norm{w_v-w_j}$, by Cauchy--Schwarz with equality when $e$ opposes the head contrast.
The unique nominal winner assumption ensures every numerator in Equation~\ref{eq:j-fidelity-radius} is positive and every denominator is nonzero.
Since $V$ is finite, all actual comparisons remain strict throughout the closed ball exactly when $E<R_v$.
At $E=R_v$, let $j$ attain the minimum and choose
\[
 e=-R_v\,\frac{w_v-w_j}{\norm{w_v-w_j}}.
\]
This produces a tie with $j$.
For $E>R_v$, increasing the magnitude of this error produces a reversal.
This proves sharpness for the full error ball.
\end{proof}

\paragraph{Raw row scales and multiple message supports.}
The unit normalization in Equation~\ref{eq:anchor} removes a token-dependent scale.
Let
\[
 g_v=(\mathcal J_R)^\top w_v,\qquad \nu_v=\norm{g_v}.
\]
For $\nu_v>0$, $g_v=\nu_v b^R_{\ell_R,v}$, and hence
\begin{equation}
 \zeta_v=\nu_v\sum_t\kappa_t
 (b^R_{\ell_R,v})^\top B_t^\top
 (G_t+\lambda I)^{-1}c_t.
 \label{eq:j-fidelity-crossgram}
\end{equation}
For $v\in I_t$, the contribution of that position reduces to
$\nu_v\kappa_t[A_tc_t]_v$, with entries indexed by the transmitted IDs.
For $v\notin I_t$, Equation~\ref{eq:j-fidelity-crossgram} retains the corresponding cross-Gram row and need not vanish.
Each message position uses its own selected support.
For $\nu_v=0$, $\zeta_v=0$, while the endpoint mismatch can contribute $w_v^\top e$ to the response.

\paragraph{Final normalization.}
For any token pair $(i,j)$, put $a=w_i-w_j$ and
$m_{ij}(y)=\ell_i(y)-\ell_j(y)$.
Subtracting Equation~\ref{eq:j-fidelity-logit} for the two tokens gives
\[
\begin{aligned}
 m_{ij}(y')-m_{ij}(y)
 &=\frac{a^\top(\widehat y-y)+a^\top e}{s'}\\
 &\quad-\frac{s'-s}{s'}m_{ij}(y).
\end{aligned}
\]
The last term can alter the sign of a margin increment even when the nominal unnormalized response is positive.
By contrast, the common denominator cancels when comparing two logits at the same endpoint, which is the ordering statement of Theorem~\ref{thm:j-fidelity}.
For a differentiable downstream map, the full local derivative is
\[
 \frac{\partial\ell_v}{\partial r_t}
 =\frac{w_v^\top}{s(y)}
 \left(I-\frac{yy^\top}{\norm y^2+d_R\varepsilon}\right)
 \frac{\partial F}{\partial r_t}(r).
\]
The derivative includes the RMSNorm scaling and radial projection; Equation~\ref{eq:j-fidelity-logit} gives the corresponding finite-write response.

\paragraph{Endpoint mismatch.}
If $F$ is continuously differentiable on the intervention segment, the fundamental theorem of calculus yields
\begin{equation}
 e=\int_0^1\sum_t
 \left[
 \frac{\partial F}{\partial r_t}(r+\theta\delta)
 -\mathcal J_R
 \right]\delta_t\,d\theta.
 \label{eq:j-fidelity-path}
\end{equation}
This is the response to the simultaneous writes at all $M$ positions.
It includes context and position differences from the calibration average as well as variation along the finite intervention path.
If
\[
 \sup_{\theta\in[0,1]}
 \left\lVert
 \frac{\partial F}{\partial r_t}(r+\theta\delta)-\mathcal J_R
 \right\rVert_{\mathrm{op}}\le\eta_t
\]
then $\norm e\le\sum_t\eta_t\norm{\delta_t}$.
Alternatively, let
$\mathcal J_{\mathrm{stack}}=[\,\mathcal J_R\ \cdots\ \mathcal J_R\,]$.
If
$\lVert DF(r)-\mathcal J_{\mathrm{stack}}\rVert_{\mathrm{op}}\le\eta$
and $DF$ is $L$-Lipschitz in operator norm along the segment, then
\[
 \norm e\le\eta\norm\delta+\tfrac12L\norm\delta^2.
\]
These bounds express the endpoint error in terms of context-dependent Jacobian variation along the intervention path.

\paragraph{Average and pathwise responses.}
For $r\in\R^2$, consider
$F(r)=(2,0)+r+(-2,0)r_1^2$.
Its Jacobians at $e_1$ and $-e_1$ average exactly to $I_2$.
At $r=e_1$ with $\delta=(0.2,0.1)$, this average predicts the increment
$(0.2,0.1)$, whereas the actual increment is $(-0.68,0.1)$.
The example shows how nonlinear path variation contributes to the endpoint mismatch $e$.

\end{document}